\documentclass[10pt]{article}
\usepackage{floatrow}

\usepackage[dvips]{graphicx}
\usepackage{algorithmic}
\usepackage{url}
\usepackage{amsmath}
\usepackage{amssymb}
\usepackage{amsthm}
\usepackage{bm}
\usepackage{Sty/mcr}
\usepackage{Sty/pkg}
\usepackage{Sty/thmE}
\usepackage{Sty/algE}

\newtheorem{definition}{Definition}
\newtheorem{theorem}{Theorem}
\newtheorem{lemma}{Lemma}
\newtheorem{corollary}{Corollary}

\newcommand{\argmax}{\mathop{\rm argmax}\limits}
\newcommand{\argmin}{\mathop{\rm argmin}\limits}

\usepackage{cite}
\usepackage{wrapfig}
\usepackage{multirow}
\usepackage{multicol}
\title{Interval-based Prediction Uncertainty Bound Computation in Learning with Missing Values}
\author{
Hiroyuki Hanada\footnotemark[1]\\
{\tt hanada.hiroyuki@nitech.ac.jp}
\and
Toshiyuki Takada\footnotemark[1]\\
{\tt takada.t.mllab.nit@gmail.com}
\and
Jun Sakuma\footnotemark[2] \footnotemark[3] \footnotemark[4]\\
{\tt jun@cs.tsukuba.ac.jp}
\and
Ichiro Takeuchi\footnotemark[1] \footnotemark[4] \footnotemark[5]\\
{\tt takeuchi.ichiro@nitech.ac.jp}
}
\date{1 March, 2018}

\begin{document}
\maketitle
\renewcommand\thefootnote{\fnsymbol{footnote}}
\footnotetext[1]{Department of Computer Science, Nagoya Institute of Technology, Nagoya, Aichi, Japan}
\footnotetext[2]{Department of Computer Science, University of Tsukuba, Tsukuba, Ibaraki, Japan}
\footnotetext[3]{CREST, Japan Science and Technology Agency, Kawaguchi, Saitama, Japan}
\footnotetext[4]{Center for Advanced Intelligence Project, RIKEN, Chuo, Tokyo, Japan}
\footnotetext[5]{Center for Materials Research by Information Integration, National Institute for Materials Science, Tsukuba, Ibaraki, Japan}
\setcounter{footnote}{0}
\renewcommand\thefootnote{\arabic{footnote}}

\begin{abstract} 
The problem of machine learning with missing values is common in many areas. 
A simple approach is to first construct a dataset without missing values simply by discarding instances with missing entries or by imputing a fixed value for each missing entry, and then train a prediction model with the new dataset. 
A drawback of this naive approach is that the uncertainty in the missing entries is not properly incorporated in the prediction. 
In order to evaluate prediction uncertainty, the multiple imputation (MI) approach has been studied, but the performance of MI is sensitive to the choice of the probabilistic model of the true values in the missing entries, and the computational cost of MI is high because multiple models must be trained.
In this paper, we propose an alternative approach called the \emph{Interval-based Prediction Uncertainty Bounding (IPUB)} method. 
The IPUB method represents the uncertainties due to missing entries as intervals, and efficiently computes the lower and upper bounds of the prediction results when all possible training sets constructed by imputing arbitrary values in the intervals are considered. 
The IPUB method can be applied to a wide class of convex learning algorithms including penalized least-squares regression, support vector machine (SVM), and logistic regression. 
We demonstrate the advantages of the IPUB method by comparing it with an existing method in numerical experiment with benchmark datasets.

\end{abstract}

\section{Introduction} \label{sec:overview}

\begin{figure*}[t]
 \includegraphics[width=\hsize,clip]{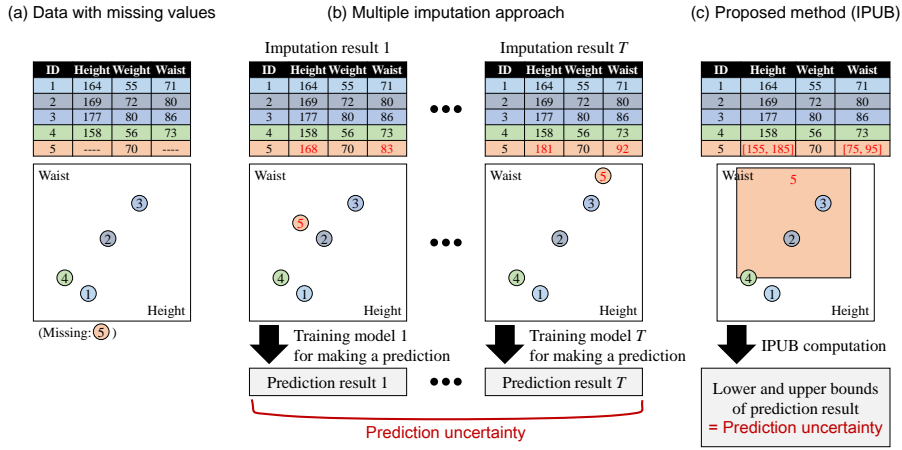}
 \caption{Schematic diagram of prediction uncertainty evaluation. (a) Given a training set with missing values, the goal is to evaluate the uncertainty in a prediction result. (b) MI constructs multiple ($T>1$) training sets by randomly imputing values based on a probabilistic model, and combines the prediction results of $T$ trained models. (c) The IPUB method represents the uncertainty in the missing entries as intervals, and estimates the prediction uncertainty by efficiently computing the lower and upper bounds of the prediction result.}
 \label{fig:Concept-Interval-Imputation}
\end{figure*}

In this paper we study the problem of learning prediction models, such as regression and classification models, from incomplete data with missing values.
This problem arises in a variety of fields ranging from traditional statistical analysis of empirical data in biomedical and social sciences to modern machine learning tasks such as natural language processing and web analytics. 
%

Many statistical and machine learning approaches for handling missing values have been developed and explored in the literature~\cite{schafer1997analysis,little2002statistical}. 
The simplest way of dealing with missing values is to discard them.
The obvious drawback of this naive approach is that the size of the training set decreases. 
Furthermore, discarding instances with missing values leads to serious biases in the prediction results unless certain simplified assumptions about the missing mechanism such as {\it the missing at random (MAR) assumption}~\cite{little2002statistical} are known to be met, which is rarely the case in practice. 
Another common approach to handling missing values is {\it single imputation}, where missing values are simply populated with specific fixed values prior to prediction modeling.
A variety of single imputation methods ranging from simple mean value imputation to advanced low-rank matrix completion~\cite{candes2009exact,candes2010power,keshavan2010matrix,cai2010singular,mazumder2010spectral} have been studied in the literature.

A principal concern with these commonly used approaches is that they cannot adequately evaluate the {\it prediction uncertainty} stemming from the uncertainty due to missing entries. 
The {\it multiple imputation (MI)} approach has been developed for evaluating prediction uncertainty~\cite{rubin1978multiple,tanner1987calculation,little1993pattern,van1999flexible,raghunathan2001multivariate,king2001analyzing,rubin2003nested,van2007multiple}.
The basic idea of MI is to introduce probabilistic models of true values in the missing entries. 
Using probabilistic models, multiple ($T > 1$) sets of values are sampled at random, and $T$ complete training sets are synthetically constructed by imputing these values to the missing entries. 
Then, $T$ prediction models are separately trained using these $T$ training sets. 
The prediction uncertainty is evaluated by combining the prediction results of these $T$ models. 
In MI, the performance of prediction uncertainty evaluations crucially depends on the choice of probabilistic models. 
Because the mechanisms of missing values are generally unknown, it is often difficult to select appropriate probabilistic models. 
Another difficulty with MI is that $T$ must be sufficiently large to adequately evaluate the prediction uncertainty. 
When the size of the training set is large, the computational cost of training multiple model can be problematic. 

In this paper, we propose a novel approach, called the {\it Interval-based Prediction Uncertainty Bounding (IPUB)} method, for evaluating the prediction uncertainty for learning prediction models such as regression or classification models with missing values.
The IPUB method represents the uncertainty in the true value of a missing entry as an interval. 
For example, if there is a missing entry for the age of a person, the IPUB method represents the uncertainty in the missing entry, e.g., as 
``{\it the missing age is in the interval $[20, 40]$}''.  
Then, for a given test input, the IPUB method also represents the uncertainty in the prediction result in the form of an interval.
For example, if the model is trained to predict the income of a person, IPUB evaluates the prediction uncertainty, e.g., as
``{\it the predicted income is in the interval [80K, 120K]}''. 

The IPUB method can be applied to a class of linear prediction models whose learning algorithm can be formulated as a certain form of penalized empirical risk minimization problem (see \eq{eq:problem-primal}) . 
This class contains many popular prediction model learning algorithms (see Table~\ref{tb:example-loss-penalty}). 
The IPUB method has two advantages. 
The first is that it does not rely on probabilistic models of the true values of the missing entries. 
To evaluate the prediction uncertainty, the IPUB method computes lower and upper bounds for the predicted value when all possible training sets constructed by imputing arbitrary values in the intervals are considered. 
%
%
The second advantage is that, unlike MI methods, the IPUB method does not require multiple model training computations. 
Specifically, using the IPUB method, a gauge of prediction uncertainty in the form of an interval can be computed only with the cost of training a single prediction model and $O(M)$ additional computations, where $M$ is the number of missing entries in the training set. 
\figurename~\ref{fig:Concept-Interval-Imputation} illustrates the difference between the MI approach and the proposed IPUB method. 


\subsection{Related works and our contributions}
The most popular MI approach is Bayesian MI~\cite{rubin1978multiple,tanner1987calculation,little1993pattern,van1999flexible,raghunathan2001multivariate}. 
In Bayesian MI, using observed entries in a training set, probabilistic models of the true values of the missing entries in the form of $p(X_u \mid X_o)$ are constructed, where $X_u$ is the set of unobserved values in the missing entries and $X_o$ is the set of observed values in the training set. 
Bayesian MI methods proceed by generating multiple ($T > 1$) imputed training sets through randomly sampling $X_u$ from $p(X_u \mid X_o)$.
Then, using each of the $T$ training sets, $T$ prediction models are trained, and the prediction uncertainty for a test input is modeled by combining the $T$ prediction results of these $T$ models. 

The main contribution of this paper is to introduce a {\it sensitivity analysis} framework for evaluating prediction uncertainty. 
Sensitivity analysis has been used for analyzing the stability of solutions to optimization problems when some parameters are perturbed \cite{saltelli2000sensitivity,bertsekas1999nonlinear}. 
Our basic idea is to compute how the prediction model (which is the optimal solution of a convex optimization problem) changes when the true values of the missing entries change. 
We developed a novel sensitivity analysis technique for efficiently computing the interval of prediction results when the true values in the missing entries change within the intervals. 
Our sensitivity analysis framework is motivated by the recent development of convex optimization techniques for sparse modeling \cite{ghaoui2012safe,wang2013lasso,wang2013scaling,ogawa2013safe,liu2014safe,wang2014safe,xiang2014screening,fercoq2015mind,ndiaye2015gap,Zimmert2015,okumura2015quick,shibagaki2015regularization,nakagawa2016safe,shibagaki2016simultaneous,Hanada2018AAAI-Modification}.

In the numerical analysis literature, so-called \emph{interval analysis} (IA) \cite{moore2009interval} has been studied, where the main goal is to evaluate the rounding error of floating point operations.
In IA, the input and output for every operation is an interval.
Therefore, some of the existing IA methods can be used similarly to the IPUB method for prediction uncertainty evaluation.
However, the output intervals for IA tend to be much wider than required because IA computes intervals for every operation, i.e., when the number of operations is large, the final intervals are much wider than the optimal ones. 
In Section \ref{sec:experiment}, we compare the IPUB method with an extant IA method~\cite{hansen2003IA}, and demonstrate that the former can provide tighter bounds than the latter in most cases. 
In addition, the IPUB method is much faster than the IA method because interval computations must be repeated many times in the latter. 

\paragraph{Notations}
For any natural number $k$, we define $[k] := \{1, \ldots, k\}$.
%
%
The $L_p$-norm of $\bm{v}\in\RR^k$ is denoted as $\|\bm{v}\|_p := [\sum_{j\in[k]} v_j^p ]^{1/p}$.
For two vectors $\bm v, \bm u \in \RR^k$, inequality notation such as $\bm v \le \bm u$ indicates elementwise inequalities,i.e. $\bm v \le \bm u \Leftrightarrow (\forall i \in [k]) v_i \le u_i$.
In addition, interval notation such as $[\bm v, \bm u]$ indicates a set of elementwise intervals $\{[v_i, u_i]\}_{i \in [k]}$.
We also use these notations for matrices.
For a condition $c$ and two statements $s_1$ and $s_2$, the notation $(c ~?~ s_1 : s_2)$ indicates ``if $c$ is true then $s_1$, otherwise $s_2$''.
For a convex function $\phi: \RR^k\to\RR$, $\phi^*(\bm{v}) := \sup_{\bm{u}\in\RR^k}\{ \bm{u}^\top\bm{v} - \phi(\bm{u}) \}$ is called the {\em convex conjugate} of $\phi$.

\section{Problem Setup and Technical Preliminaries}

In this section we configure the problem and provide technical preliminaries for convex analysis.

\subsection{Problem Setup} \label{subsec:problem_setup}

Consider learning a linear prediction model $f$ that maps a $d$-dimensional input vector $\bm x \in \cX \subseteq \RR^d$ to a scalar output $y \in \cY \subseteq \RR$ where $\cX$ and $\cY$ are the input and output domain, respectively. 
The prediction model is written as 
\begin{align*}
 y = g(\bm w^\top \bm x), 
\end{align*}
where $g: \RR \to \cY$ is a monotonically non-decreasing function that maps the linear model output in $\RR$ to the output in $\cY$, and $\bm w \in \RR^d$ is the vector of linear model parameters. 
For example, $\cY = \RR$ and $g$ is the identity function in the case of least-squares regression, $\cY = \{\pm 1\}$ and $g$ is the sign function in the case of SVM, and $\cY = [0, 1]$ and $g$ is the sigmoid function in the case of logistic regression\footnote{
In logistic regression, the domain of the training output $y_i$ is $\{\pm 1\}$ as in SVM, while the domain of the prediction results $g(\bm w^\top \bm x)$ can be considered as $[0, 1]$ when we use logistic regression model for predicting the probability of the class label being $+1$. 
}. 

Let $(X, \bm y)$ be the training set, where $X$ is an $n$-by-$d$ input matrix and $\bm y$ is an $n$-dimensional output vector. 
The $i$-th row, the $j$-th column and the $(i, j)$-th element of $X$ are respectively written as $\bm x_{i \cdot}$, $\bm x_{\cdot j}$ and $x_{ij}$, and the $i$-th element of $\bm y$ is written as $y_i$ for $i \in [n]$ and $j \in [d]$.
In this paper, we consider a situation where some of the elements in $X$ are missing.
The set of missing elements is denoted as $\cM \subset \{(i, j)\}_{i \in [n], j \in [d]}$, and its size is written as $M=|\cM|$. 
We denote the true input matrix as $X^*$, and its $i$-th row, $j$-th column, $(i,j)$-th element are similarly written as $\bm x^*_{i \cdot}$, $\bm x^*_{\cdot j}$, $x^*_{ij}$, respectively. 
Note that we cannot actually observe $x^*_{ij}, (i, j) \in \cM$, while the remaining elements of $X^*$ are the same as the corresponding elements of $X$, i.e., $x^*_{ij} = x_{ij} \text{ if } (i, j) \notin \cM$.

For now, consider a hypothetical situation where we can observe $X^*$.
We use a class of learning algorithms formulated as a penalized empirical risk minimization problem in the form of 
\begin{align}
 \label{eq:problem-primal}
 \bm w^* := \argmin_{\bm w \in \RR^d} \frac{1}{n} \sum_{i \in [n]} \ell(y_i, \bm w^\top \bm x^*_i) + \rho(\bm w),
\end{align}
where $\ell: \cY \times \RR \to \RR_+$ is a convex loss function, and $\rho: \RR^d \to \RR_+$ is a convex penalty function. 
The class \eq{eq:problem-primal} contains many well-known learning algorithms. 

When the training input matrix $X$ contains missing entries, we cannot obtain the optimal model parameters $\bm w^*$. 
A commonly taken heuristic approach for circumventing this difficulty is to construct a synthetic input matrix $\hat{X}$ either by discarding instances that contain missing entries or by imputing fixed values to the missing entries, and then solve the problem in \eq{eq:problem-primal} by using $(\hat{X}, \bm y)$ as the training set instead of $(X^*, \bm y)$. 
As discussed in Section \ref{sec:overview}, this heuristic approach completely ignores the effect of the uncertainty in missing entries.

As noted, in this paper, we represent the uncertainty in a missing entry as an interval.
For a missing entry $(i, j) \in \cM$, let $[\underline{x_{ij}}, \overline{x_{ij}}]$ represent the interval in which the true value $x^*_{ij}$ is located, where $\underline{x_{ij}} < \overline{x_{ij}}$ are the smallest and largest possible value that $x^*_{ij}$ takes, respectively. 
We also denote the uncertainty in missing entries in the entire input matrix $X$ as $[\underline{X}, \overline{X}]$, where, for an observed element $(i, j) \notin \cM$, $\underline{x_{ij}} = \overline{x_{ij}} = x_{ij}$. 
When we write $X^\prime \in [\underline{X}, \overline{X}]$, $X^\prime$ represents an arbitrary matrix such that all the elements satisfy $x^\prime_{ij} \in [\underline{x_{ij}}, \overline{x_{ij}}]$. 

Our goal is to fully incorporate the uncertainty in missing entries for evaluating the uncertainty in prediction results. 
To this end, let us consider a set of trained model parameters
\begin{align}
 \label{eq:set_of_parameter}
 \cW^* := \{\argmin_{\bm w \in \RR^d} \frac{1}{n} \ell(y_i, \bm w^\top {\bm x}^\prime_{i \cdot}) + \rho(\bm w) \mid X^\prime \in [\underline{X}, \overline{X}]\}.
\end{align}
Then, for a test input $\bm x \in \cX$, the uncertainty in the prediction result is also represented as an interval
\begin{align*}
 g(\bm w^{*\top} \bm x) \in [\underline{g(\bm w^{*\top} \bm x)}, \overline{g(\bm w^{*\top} \bm x)}], 
\end{align*}
where
\begin{subequations} 
\label{eq:prediction_bounds}
\begin{align}
 \underline{g(\bm w^{*\top} \bm x)} &:= \mathop{\rm inf}_{\bm w \in \cW^*} g(\bm w^\top \bm x), \\
 \overline{g(\bm w^{*\top}\bm x)} &:= \mathop{\rm sup}_{\bm w \in \cW^*} g(\bm w^\top \bm x).
\end{align}
\end{subequations}
Namely,
the set of parameters $\cW^*$ in \eq{eq:set_of_parameter} represents the collection of all possible solutions of the penalized empirical risk minimization problems in \eq{eq:problem-primal} when the input matrix $X^\prime$ is arbitrarily chosen from the interval $[\underline{X}, \overline{X}]$, 
while 
$\underline{g(\bm w^{*\top} \bm x)}$
and 
$\overline{g(\bm w^{*\top} \bm x)}$
in \eq{eq:prediction_bounds}
respectively indicate the smallest and largest possible prediction result when the uncertainty in missing entries is incorporated. 

Unfortunately, it is not possible to obtain $\cW^*$ in practice because, by definition, it requires us to solve infinitely many penalized empirical risk minimization problems. 
Our fundamental approach to tackling this difficulty involves developing a computationally feasible method that can efficiently compute 
\begin{subequations} 
 \begin{align}
  \label{eq:bound_prediction_L}
  &L(\underline{g(\bm w^{*\top} \bm x)}) \le \underline{g(\bm w^{*\top} \bm x)} \text{ and }\\
  \label{eq:bound_prediction_U}
  &U(\overline{g(\bm w^{*\top} \bm x)}) \ge \overline{g(\bm w^{*\top} \bm x)}, 
 \end{align}
\end{subequations}
where $L(\underline{g(\bm w^{*\top} \bm x)})$ and $U(\overline{g(\bm w^{*\top} \bm x)})$ are a lower bound of the smallest possible prediction result and an upper bound of the largest possible prediction result, respectively.\footnote{
In the case of binary classification, e.g., by SVM, the uncertainty in the prediction results for the output domain $\cY = \{\pm 1\}$ is given by considering a lower and an upper bound of the linear discriminant function from
$L(\underline{\bm w^{*\top} \bm x})$
and 
$U(\overline{\bm w^{*\top} \bm x})$. 
The uncertainty in the predicted labels in $\{\pm 1\}$ is then given by
\begin{align*}
 g(\bm w^{*\top} \bm x) = \mycase{
 +1 & \text{ if } L(\underline{\bm w^{*\top} \bm x)}) > 0, \\
 -1 & \text{ if } U(\overline{\bm w^{*\top} \bm x)}) < 0, \\
 \text{unknown} & \text{ otherwise}. 
 }
\end{align*}
}
If these bounds are sufficiently tight, we can use them for evaluating the uncertainty in the prediction results. 

\subsection{Technical Preliminaries for Convex Analysis} \label{subsec:technice_preliminaries}
%
Here, we present technical preliminaries for convex analysis which are used for constructing the proposed method in the next section. 
For a comprehensive introduction to convex analysis, see, e.g., \cite{Boyd04a}. 
In this subsection, for notational simplicity, we denote the input matrix as $X$, and assume that there are no missing entries in $X$. 

Let us denote the objective function of the empirical risk minimization problem in \eq{eq:problem-primal} as 
\begin{align}
 P_{X}(\bm w) := \frac{1}{n} \sum_{i \in [n]} \ell(y_i, \bm w^\top \bm x_i) + \rho(\bm w). 
\end{align}
The method proposed in the next section can be applied to the problem in \eq{eq:problem-primal}
when 
the loss function
$\ell$
and
the penalty function
$\rho$
satisfy certain conditions. 

The {\em dual problem} of
\eq{eq:problem-primal} is written, 
by using the convex conjugates $\ell^*$ and $\rho^*$\footnote{For $\ell$, we take the convex conjugate with respect to the second argument.}, 
as
\begin{align}
 & \textstyle
 \bm \alpha^*
 :=
 \argmax_{\bm \alpha \in {\rm dom} \ell^*, \frac{1}{n} X^\top \bm \alpha \in {\rm dom} \rho^*}
 D_X(\bm \alpha),
\quad
 \text{where}
 \nonumber \\
 & \textstyle
 D_X(\bm \alpha)
 := - \frac{1}{n} \sum_{i \in [n]} \ell^*(y_i, -\alpha_i) - \rho^*(\frac{1}{n} X^\top \bm \alpha),
\label{eq:problem-dual}
\end{align}
where ${\rm dom}\cdot$ denotes the domain of the function.
See, for example, Corollary 31.2.1 in \cite{rockafellar1970convex} for the details of dual problem derivation. 
Between the primal optimal solution
$\bm w^*$
and the dual optimal solution
$\bm \alpha^*$, 
the following relationships are known to hold under certain regularity conditions: 
\begin{align}
& \bm{w}^* \in \left.
(\partial/\partial \bm {v})\rho^*(\bm{v})\right|_{\bm{v} = \frac{1}{n}X^\top\bm{\alpha}^*},&&\text{({\em KKT condition})} \label{ex:KKT-primal}\\
& \alpha^*_i \in -\left.(\partial/\partial \bm {v})\ell(y_i, v)\right|_{v = \bm{w}^{*\top}\bm{x}_{i\cdot}},&&\text{({\em KKT condition})} \label{ex:KKT-dual}\\
& D_X(\bm \alpha^*) = P_X(\bm{w}^*). &&\text{({\em strong duality})} \label{ex:strong-duality}
\end{align}
The method proposed in the next section can be applied to any convex loss functions such as the squared loss function for regression, the hinge loss function for SVM, and the logistic loss function for logistic regression. 
See Table~\ref{tb:example-loss-penalty} for these loss function specifications, their convex conjugates, and their subderivatives. 

On the other hand, the penalty function $\rho$ must satisfy the following properties.
\begin{definition}[$\lambda$-strongly convex decomposable penalty] \label{df:strongly-convex-decomposable penalty}
A penalty function
$\rho$
is said to be
\emph{$\lambda$-strongly convex} for a
$\lambda > 0$ 
if 
 $\rho(\bm{w}_1) - \rho(\bm{w}_2)
 \geq
 \partial \rho(\bm{w}_1)^\top (\bm{w}_1 - \bm{w}_2) +
 (\lambda/2)\|\bm{w}_1 - \bm{w}_2\|_2^2$
holds for any
 $\bm{w}_1, \bm{w}_2 \in \RR^d$.
 Furthermore, 
 it is said to be \emph{decomposable}
 if it is written as
\begin{align*}
 \rho(\bm w) = \sum_{j \in [d]} \rho_j(w_j),
 \end{align*}
where $\rho_j: \RR \to \RR$, $j \in [d]$, is a convex function. 
\end{definition}
Definition \ref{df:strongly-convex-decomposable penalty} includes the $L_2$ penalty, the elastic net penalty, and some other popular penalties. 
See Table~\ref{tb:example-loss-penalty} for these penalty function specifications, their convex conjugates, and their subderivatives. 

\begin{table*}[tp]
\caption{Examples of loss functions $\ell$ and penalty functions $\rho$ for which the proposed IPUB method is applicable}
\label{tb:example-loss-penalty}
\begin{center}
{\small
\begin{tabular}{ccc}
\hline
\hline
Loss function $\ell(y_i, \bm w^\top\bm x_{i\cdot} )$ & $\ell^*(y_i, -\alpha_i )$ & $\displaystyle \left.\frac{\partial}{\partial v}\ell(y_i, v)\right|_{v = \bm{w}^{*\top}\bm{x}_{i\cdot}}$ \\
\hline
\hline
	$\displaystyle \begin{matrix}
	(y_i - \bm w^\top \bm x_{i\cdot})^2 \\
	(\text{{\em Squared loss}};\\
	y_i\in\RR)
	\end{matrix}$
	& $\displaystyle \frac{1}{4} \alpha_i(\alpha_i - 4 y_i)$\quad($\displaystyle \alpha_i\in\RR$)
	& $\displaystyle 2(\bm{w}^{*\top}\bm{x}_{i\cdot} - y_i)$
	\\
\hline
	$\displaystyle \begin{matrix}
	[1 - y_i \bm w^\top \bm x_{i\cdot}]_+~{}^{\text{*1}} \\
	(\text{{\em Hinge loss}};\\
	y_i\in\{-1, +1\})
	\end{matrix}$
	& $\displaystyle -\frac{\alpha_i}{y_i}$\quad($\displaystyle 0\leq \frac{\alpha_i}{y_i}\leq 1$)
	& $\begin{cases}
		0 & (\text{if $y_i \bm{w}^{*\top}\bm{x}_{i\cdot} > 1$}) \\
		[-y_i, 0]~{}^{\text{*2}} & (\text{if $y_i \bm{w}^{*\top}\bm{x}_{i\cdot} = 1$}) \\
		-y_i & (\text{if $y_i \bm{w}^{*\top}\bm{x}_{i\cdot} < 1$})
		\end{cases}$
	\\
\hline
	$\displaystyle \begin{matrix}
	\log(1 + \exp(- y_i \bm w^\top \bm x_{i\cdot} )) \\
	(\text{{\em Logistic loss}};\\
	y_i\in\{-1, +1\})
	\end{matrix}$
	& $\displaystyle \begin{matrix}
	\displaystyle \left(1 - \frac{\alpha_i}{y_i}\right)\log|y_i - \alpha_i|
	\\ \displaystyle + \frac{\alpha_i}{y_i}\log|\alpha_i| - \log|y_i|
	\\ \text{($\displaystyle 0\leq \frac{\alpha_i}{y_i}\leq 1$)~${}^{\text{*3}}$} 
	\end{matrix}$
	& $\displaystyle -\frac{y_i}{1 + \exp(y_i\bm{w}^{*\top}\bm{x}_{i\cdot})}$
	\\
\hline
\hline
\\
\hline
\hline
Penalty function $\rho(\bm{w})$ & $\displaystyle \rho^*(\frac{1}{n} X^\top \bm \alpha)$ & $\displaystyle \left.\frac{\partial}{\partial\bm{v}}\rho^*(\bm{v})\right|_{\bm{v} = \frac{1}{n}X^\top\bm{\alpha}^*}$\\
\hline
\hline
	$\displaystyle \begin{matrix}
	\displaystyle\frac{\lambda}{2}\|\bm w\|_2^2$~${}^{\text{*4}} \\
	(\text{$L_2$~{\em penalty}})
	\end{matrix}$
	& $\displaystyle \frac{1}{2\lambda n^2}\|X^\top \bm \alpha\|_2^2$
	& $\displaystyle \frac{1}{\lambda n}X^\top \bm \alpha^*$
	\\
\hline
	$\displaystyle \begin{matrix}
	\displaystyle \frac{\lambda}{2}\|\bm w\|_2^2 + \kappa\|\bm w\|_1~{}^{\text{*4}} \\
	(\text{{\em Elastic net penalty}})
	\end{matrix}$
	& $\displaystyle \begin{matrix}
	\displaystyle \frac{1}{2\lambda n^2}\times \\
	\displaystyle \sum_{j\in[d]}([|\bm \alpha^\top\bm{x}_{\cdot j}| - n\kappa]_+)^2$~${}^{\text{*1}}
	\end{matrix}$
	& $\displaystyle \begin{matrix}
	\displaystyle \left[ \frac{[F_j]_+ - [G_j]_+}{\lambda n}\right]_{j\in[d]}~{}^{\text{*1}}
	\\ (F_j = \bm{\alpha}^{*\top}\bm{x}_{\cdot j} - n\kappa,
	\\ G_j = -\bm{\alpha}^{*\top}\bm{x}_{\cdot j} - n\kappa)
	\end{matrix}$
	\\
\hline
\hline
\end{tabular}
\\
\begin{minipage}[c]{1.0\textwidth}
	*1: $[t]_+ := \max\{0, t\}$.
	\quad
	*2: If $y_i = -1$, replace $[-y_i, 0]$ with $[0, -y_i]$.
	\quad
	*3: If $\frac{\alpha_i}{y_i} = 0$ or $1$, treat $\ell^*(y_i, -\alpha_i ) = 0$.
	\quad
	*4: $\lambda, \kappa > 0$ are tuning parameters.
\end{minipage} 
}
\end{center}
\end{table*}

\section{Proposed Method} \label{sec:interval-privacy-convex}

In this section, we present our main results.
Proofs of the theorems and the corollary in this section are all presented in Appendix \ref{sec:proof}.

Given a test input $\bm x \in \cX$, our goal is to evaluate the prediction uncertainty in the form of an interval
\begin{align*}
 g(\bm w^{*\top} \bm x) \in [L(\underline{g(\bm w^{*\top} \bm x)}), U(\overline{g(\bm w^{*\top} \bm x)})],
\end{align*}
where
$L(\underline{g(\bm w^{*\top} \bm x)})$
and 
$U(\overline{g(\bm w^{*\top} \bm x)})$
are the lower bound of the smallest possible prediction result 
and 
the upper bound of the largest possible prediction result,
respectively. 
Our approach with the proposed IPUB method involves the following two steps:
\\

\noindent
{\bf step 1}. Compute a superset $\cW \supseteq \cW^*$ from $\underline{X}$, $\overline{X}$ and $\bm y$. 
\\

\noindent
{\bf step 2}. Compute $L(\underline{g(\bm w^{*\top} \bm x)})$ and $U(\overline{g(\bm w^{*\top} \bm x)})$ from $\cW$. 
\\

\noindent
As described below, the computational cost of step 1 is the cost of solving a single penalized empirical risk minimization problem in \eq{eq:problem-primal} and $O(M)$ additional computations where $M$ is the number of missing entries.
Once we compute a superset $\cW \supseteq \cW^*$, the computational cost of step 2 is $O(d)$ for each test input $\bm x \in \cX \subseteq \RR^d$.

\subsection{Step 1}
\begin{algorithm}[tp]
\caption{Step 1}
\label{alg:compute-duality-gap}
\begin{algorithmic}[1]
 \STATE {\bf Input} $\underline{X}, \overline{X}, \bm y$
 \STATE
 Select an arbitrary $X^\prime \in [\underline{X}, \overline{X}]$
 \STATE
 Compute $\bm w^\prime$, $\bm \alpha^\prime$ as in  \eq{eq:sol} by using $X^\prime$
 \STATE
 $\Delta \gets 0$.
 \FORALL{$i\in{\cI}$}
 \STATE $p^-_i\gets \bm w^{\prime \top}\bm x^\prime_{i\cdot}$,
 $p^+_i\gets \bm w^{\prime \top}\bm x^\prime_{i\cdot}$.
 \FORALL{$j\in\cJ_\cM(i)$}
 \STATE $p^-_i\gets p^-_i + (w^\prime_j > 0 ~?~ w^\prime_j \underline{x_{ij}} : w^\prime_j \overline{x_{ij}}) - w^\prime_jx^\prime_{ij}$.
 \STATE $p^+_i\gets p^+_i + (w^\prime_j > 0 ~?~ w^\prime_j \overline{x_{ij}} : w^\prime_j \underline{x_{ij}}) - w^\prime_jx^\prime_{ij}$.
 \ENDFOR
 \STATE $\Delta\gets\Delta+\frac{1}{n}[\max\{ \ell(y_i, p^-_i), \ell(y_i, p^+_i) \} - \ell(y_i, \bm w^{\prime\top}\bm x^\prime_{i\cdot})]$
 \ENDFOR
 \FORALL{$j\in{\cJ}$}
 \STATE $q^-_j\gets \bm \alpha^{\prime \top}\bm x^\prime_{\cdot j}$,
 $q^+_j\gets \bm \alpha^{\prime \top}\bm x^\prime_{\cdot j}$.
 \FORALL{$i\in\cI_\cM(j)$}
 \STATE $q^-_j\gets q^-_j + (\alpha^\prime_i > 0 ~?~ \alpha^\prime_i \underline{x_{ij}} : \alpha^\prime_i \overline{x_{ij}}) - \alpha^\prime_ix^\prime_{ij}$
 \STATE $q^+_j\gets q^+_j + (\alpha^\prime_i > 0 ~?~ \alpha^\prime_i \overline{x_{ij}} : \alpha^\prime_i \underline{x_{ij}}) - \alpha^\prime_ix^\prime_{ij}$
 \ENDFOR
 \STATE $\Delta\gets\Delta + \max\{ \rho^*_j(\frac{1}{n} q^-_j), \rho^*_j(\frac{1}{n} q^+_j) \} - \rho^*_j(\frac{1}{n}\bm \alpha^{\prime \top}\bm x^\prime_{\cdot j})$
 \ENDFOR
 \STATE
 {\bf Return} $\cW$, i.e., the center $\bm w^\prime$ and the radius $\sqrt{2 \Delta/\lambda}$
\end{algorithmic}
\end{algorithm}
The pseudo-code for step 1 is presented in Algorithm~\ref{alg:compute-duality-gap}. 
First, we select an arbitrary input matrix $X^\prime$ from the interval, i.e., 
\begin{align*}
 X^\prime \in [\underline{X}, \overline{X}].
\end{align*}
Next, we compute the primal and dual solutions by solving the penalized empirical risk minimization problem in \eq{eq:problem-primal} by using $(X^\prime, \bm y)$ as the training set, i.e., 
\begin{subequations}
 \label{eq:sol}
  \begin{align}
   \label{eq:sol_primal}
   \bm w^\prime &= \argmin_{\bm w} P_{X^\prime}(\bm w), 
   \\
   \label{eq:sol_dual}
   \bm \alpha^\prime &= \argmax_{\bm \alpha} D_{X^\prime}(\bm \alpha).
  \end{align}
\end{subequations}
Then, we compute a superset $\cW \supseteq \cW^*$ as 
\begin{align}
 \label{eq:cW}
 \cW := \{\bm w \in \RR^d \mid \|\bm w - \bm w^\prime \|_2 \le \sqrt{2 \Delta/\lambda }\}, 
\end{align}
where
\begin{align}
 \label{eq:duality_gap}
 \Delta
 :=
 \max_{X^{\prime \prime} \in [\underline{X}, \overline{X}]} P_{X^{\prime \prime}}(\bm w^\prime)
 -
 \min_{X^{\prime \prime} \in [\underline{X}, \overline{X}]} D_{X^{\prime \prime}}(\bm \alpha^\prime).
\end{align}

The following theorem formally states that $\cW$ in \eq{eq:cW} is actually a superset of $\cW^*$ in \eq{eq:set_of_parameter}. 
\begin{theorem} \label{th:interval-primal-solution}
Suppose that the penalty function $\rho$ satisfies the properties in Definition~\ref{df:strongly-convex-decomposable penalty}.
Let $X^\prime \in [\underline{X}, \overline{X}]$ be any input matrix in the interval. 
Furthermore, 
let 
$\bm{w}^\prime$
and
$\bm{\alpha}^\prime$
be the primal and dual solutions of the problem in \eq{eq:problem-primal}
when the input matrix is $X^\prime$ 
as defined in
\eq{eq:sol_primal} and 
\eq{eq:sol_dual}, respectively. 
Then,
the set of model parameters $\cW$
defined in 
\eq{eq:cW}
is a superset of
$\cW^*$,
i.e.,
$\cW \supseteq \cW^*$.
\end{theorem}
Theorem \eqref{th:interval-primal-solution} is based on the strong convexity
(Definition~\ref{df:strongly-convex-decomposable penalty}) of the primal objective function $P$.
By strong convexity, for any $X^{\prime\prime}\in[\underline{X}, \overline{X}]$,
the difference between two optimal solutions $\|\bm{w}^\prime - \bm{w}^{\prime\prime}\|_2$
can be bounded from above by the difference between the objective function values
$P_{X^{\prime\prime}}(\bm{w}^\prime) - P_{X^{\prime\prime}}(\bm{w}^{\prime\prime})$,
which can be bounded by $\Delta$.
See Lemma \ref{lm:strong-convexity-sphere} in Appendix \ref{sec:proof} for details.

Next, we observe that, given the solution $\bm w^\prime$, the cost of computing the superset $\cW$ in \eq{eq:cW} is only $O(M)$.
Specifically, in the following theorem we state that $\Delta$ in \eq{eq:duality_gap} can be computed with a cost of $O(M)$ if some relevant quantities have already been computed. 
%
\begin{theorem} \label{th:interval-primal-computation}
Suppose that the same condition as Theorem \ref{th:interval-primal-solution} holds. 
Furthermore, assume that the following values have already been computed:
 \begin{subequations} 
\label{eq:precomputed}
 \begin{align}
 \label{eq:precomputed1}
 &
 \bm w^{\prime \top} \bm x^\prime_{i \cdot} ~ \forall i \in [n], 
 ~
 \bm \alpha^{\prime \top} \bm x^\prime_{\cdot j} ~ \forall j \in [d],
 \\
 \label{eq:precomputed2}
 &
 \cI := \{i \in [n] \mid (i, j) \in \cM\}, 
 \\
 \label{eq:precomputed3}
 &
 \cJ := \{j \in [d] \mid (i, j) \in \cM\}, 
 \\
 \label{eq:precomputed4}
 &
 \cI_{\cM}(j) := \{i \in [n] \mid (i, j) \in \cM\} ~ \forall j \in \cJ, 
 \\
 \label{eq:precomputed5}
 &
 \cJ_{\cM}(i) := \{j \in [d] \mid (i, j) \in \cM\} ~ \forall i \in \cI.
 \end{align}
 \end{subequations}
Then $\Delta$ in \eq{eq:duality_gap} is computed with a cost $O(M)$ as
\begin{align}
\Delta = & \frac{1}{n} \sum_{i\in{\cI}} [ \max\{ \ell(y_i, p^-_i), \ell(y_i, p^+_i) \} - \ell(y_i, \bm{w}^{\prime \top}\bm x^\prime_{i\cdot}) ]
\nonumber
 \\
& \! \! + \sum_{j\in{\cJ}} [ \textstyle \max\{ \rho^*_j(\frac{1}{n} q^-_j), \rho^*_j(\frac{1}{n} q^+_j) \} - \rho^*_j(\frac{1}{n} \bm{\alpha}^{\prime \top}\bm x^\prime_{\cdot j}) ],
\label{eq:duality-gap-upper-bound-limited}
\end{align}
where
\begin{align*}
p^-_i := & \bm w^{\prime\top}\bm x^\prime_{i\cdot} + \sum_{j\in\cJ_{\cM}(i)}[(w^\prime_j > 0 ~?~ w^\prime_j \underline{x_{ij}} : w^\prime_j \overline{x_{ij}}) - w^\prime_j x^\prime_{ij}], \\
p^+_i := & \bm w^{\prime\top}\bm x^\prime_{i\cdot} + \sum_{j\in\cJ_{\cM}(i)}[(w^\prime_j > 0 ~?~ w^\prime_j \overline{x_{ij}} : w^\prime_j \underline{x_{ij}}) - w^\prime_j x^\prime_{ij}], \\
q^-_j := & \bm \alpha^{\prime\top}\bm x^\prime_{\cdot j} + \sum_{i\in\cI_{\cM}(j)}[(\alpha^\prime_i > 0 ~?~ \alpha^\prime_i \underline{x_{ij}} : \alpha^\prime_i \overline{x_{ij}}) - \alpha^\prime_i x^\prime_{ij}], \\
q^+_j := & \bm \alpha^{\prime\top}\bm x^\prime_{\cdot j} + \sum_{i\in\cI_{\cM}(j)}[(\alpha^\prime_i > 0 ~?~ \alpha^\prime_i \overline{x_{ij}} : \alpha^\prime_i \underline{x_{ij}}) - \alpha^\prime_i x^\prime_{ij}].
\end{align*}
\end{theorem}
\noindent
Note that
the assumption that the values in \eq{eq:precomputed} have been precomputed is not problematic
because
the quantities in \eq{eq:precomputed1} must have been computed when solving the optimization problem for computing $\bm w^\prime$ and $\bm \alpha^\prime$, while the quantities in \eq{eq:precomputed2} to \eq{eq:precomputed5} can be known beforehand because they are merely lists of the positions in the missing entries. 
In addition, the memory required to store all these values is only $O(n + d + M)$. 

\subsection{Step 2}
Because the superset $\cW \subseteq \cW^*$ in step 1 is a sphere in $\RR^d$, it is straightforward to compute
a lower bound
$L(\underline{g(\bm w^{*\top} \bm x)})$
and
an upper bound
$U(\overline{g(\bm w^{*\top} \bm x)})$
as in the following corollary.
\begin{corollary} \label{co:bound-inner-product}
 Assume that a set of model parameters $\cW$ is a superset of $\cW^*$ and is represented as a sphere with a center $\bm w^\prime$ and a radius $\sqrt{2 \Delta/\lambda}$, then 
 the prediction result
 $g(\bm w^{*\top} \bm x)$
 for a test input
 $\bm x \in \cX \subseteq \RR^d$
 is guaranteed to be within the following interval
 \begin{align}
  \label{eq:prediction_interval}
  g(\bm w^{*\top} \bm x) \in [L(\underline{g(\bm w^{*\top} \bm x)}), U(\overline{g(\bm w^{*\top} \bm x)})], 
 \end{align}
 where 
  \begin{align*}
   L(\underline{g(\bm w^{*\top} \bm x)}) &= \inf_{\bm w \in \cW} g(\bm w^\top \bm x) = g(\bm w^{\prime \top} \bm x - \|\bm x\|_2 \sqrt{2 \Delta/\lambda}), \\
   U(\overline{g(\bm w^{*\top} \bm x)}) &= \sup_{\bm w \in \cW} g(\bm w^\top \bm x) = g(\bm w^{\prime \top} \bm x + \|\bm x\|_2 \sqrt{2 \Delta/\lambda}).
  \end{align*}
\end{corollary}
\noindent
The cost of computing the prediction interval in \eq{eq:prediction_interval} is $O(d)$.

\section{Experiment} \label{sec:experiment}
%
We explored the performance of the proposed IPUB method by comparing it with the Interval Newton (INewton) method \cite{hansen2003IA}.
We applied the two methods to three datasets taken from the UCI machine learning repository \cite{Asuncion07}, where approximately 90\% of each dataset was used for the training set, while the remaining 10\% was used as the test set (see Table \ref{tb:datasets-info}). Detailed setups are presented in Appendix \ref{ch:data-setups}.
For brevity, we only present results for the logistic regression in which we predict the probability of the class label being $+1$ by using the sigmoid function $g(\bm w^{*\top} \bm x) = 1/(1 + \exp(-\bm w^{*\top} \bm x))$. 
We compared the performance of the two methods in terms of prediction uncertainty interval length and computation time. 
The former is evaluated as the difference between the upper bound and the lower bound of the predicted probability, i.e., $U(\overline{g(\bm w^{*\top} \bm x)}) - L(\underline{g(\bm w^{*\top} \bm x)})$.
All the computations were carried out with an Intel Xeon CPU (3.10GHz clock) and 256GB RAM.

\begin{table}[t]
\caption{Datasets used for experimentation. All are taken from the UCI Machine Learning Repository \cite{Asuncion07}.}
\label{tb:datasets-info}
\begin{center}
\begin{tabular}{cccc}
\hline
        & \multicolumn{2}{c}{Number of} & \\
        & \multicolumn{2}{c}{instances} & Number of \\
\cline{2-3}
Dataset & Training & Test               & features \\
\hline
{\tt default} {\scriptsize (default of credit card clients)} &    27,000 &  3,000 & 28 \\
{\tt drive} {\scriptsize (Sensorless Drive Diagnosis)}       &    52,658 &  5,851 & 48 \\
{\tt HIGGS}                                                  & 1,100,000 & 11,000 & 28 \\
\hline
\end{tabular}
\end{center}
\end{table}

\subsection{Settings}
\paragraph{Experimental Setting Parameters}
We considered various experimental settings by adjusting three parameters. 
The first parameter is the missing-value proportion $b \in \{0.01, 0.001\}$; we randomly chose $n \times d \times b$ elements of the input matrix and regard them as missing entries. 
The second parameter is the missing value interval length $\alpha \in \{0.5, 0.9\}$; we set the interval for each missing element to [$(1-\alpha)/2$-quantile, $(1+\alpha)/2$-quantile] of the original values in the element. 
The third parameter is the penalty parameter $\lambda \in \{0.1, 1.0\}$.

\paragraph{Implementation of IPUB}
In the implementation of Algorithm \ref{alg:compute-duality-gap},
$X^\prime$ is chosen as $(\underline{X} + \overline{X})/2$,
$\bm{w}^\prime$ is computed using the trust region Newton method
implemented in LIBLINEAR \cite{Fan08b}, and
$\bm{\alpha}^\prime$ is computed from $\bm{w}^\prime$ by \eqref{ex:KKT-dual}.

\paragraph{Interval Newton (INewton) Method}
We compared the performances of the proposed IPUB method with an interval analysis method called the Interval Newton (INewton) method \cite{hansen2003IA}.
In the INewton method, every operation in each Newton step is conducted for parameters and data represented as intervals. 
The computational cost of the INewton method is $O(Knd^3)$, where $K$ is the number of Newton iteration steps; the cost is proportional to $d^3$ due to the matrix inverse operation in each Newton step. 

\subsection{Results}

\paragraph{Prediction Uncertainty Interval Length}
Figure~\ref{fig:uncertainty-histogram} shows results pertaining to the prediction uncertainty interval length.
Each plot contains a green histogram of prediction uncertainty interval lengths $U(\overline{g(\bm w^{*\top} \bm x)}) - L(\underline{g(\bm w^{*\top} \bm x)})$ by the IPUB method, and a purple histogram of the counterpart prediction uncertainty interval lengths obtained using the INewton method for each test input $\bm x$. 
In most cases, the prediction uncertainty interval lengths associated with the IPUB method are much smaller than those for the INewton method.
Because both methods guarantee that the predicted values are within the intervals, the results testify to the superiority of the IPUB method. 

\begin{figure*}[tp]
\begin{center}
\includegraphics[width=\hsize,clip]{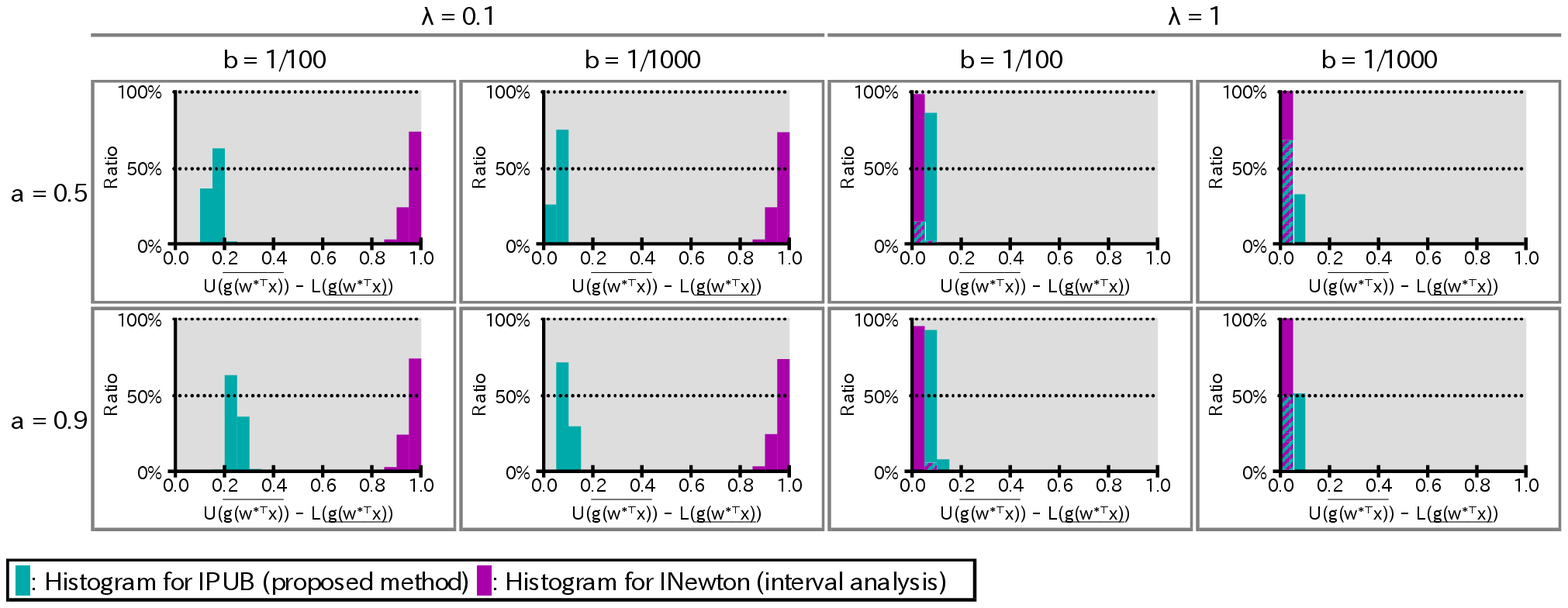}\\
(a) {\tt default} dataset \\
\includegraphics[width=\hsize,clip]{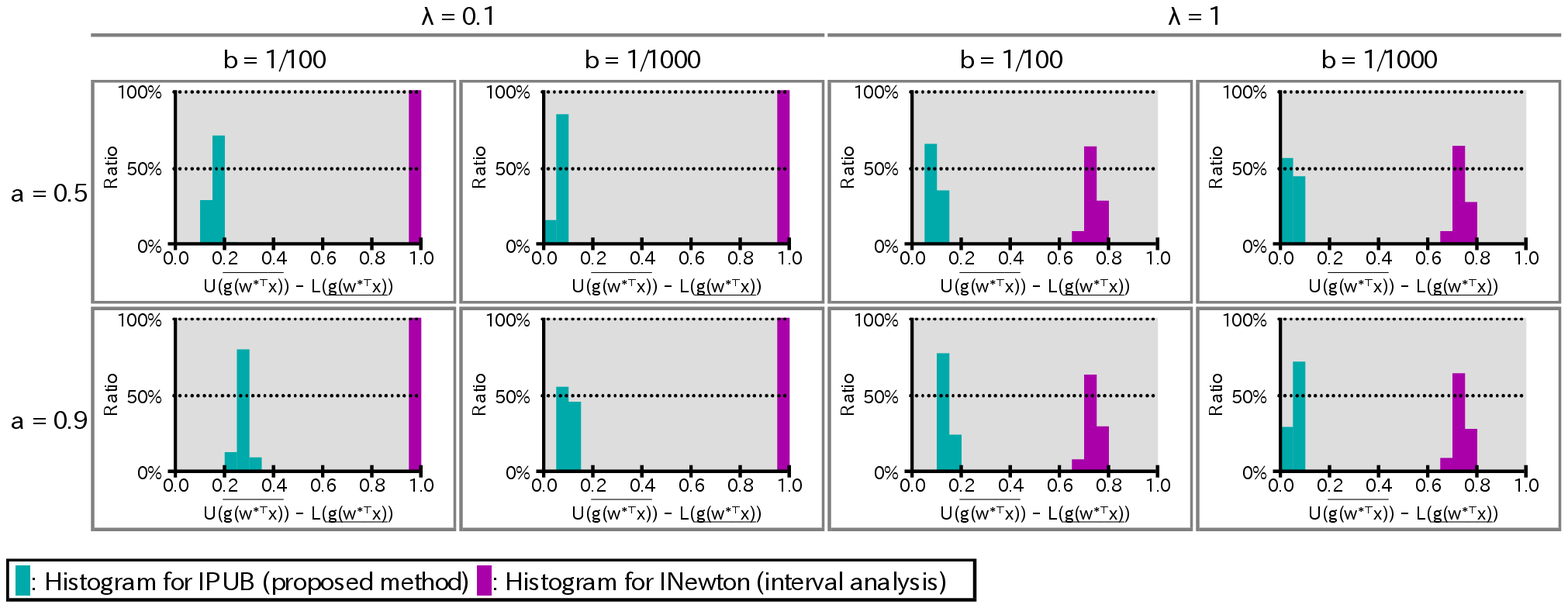} \\
(b) {\tt drive} dataset \\
\includegraphics[width=\hsize,clip]{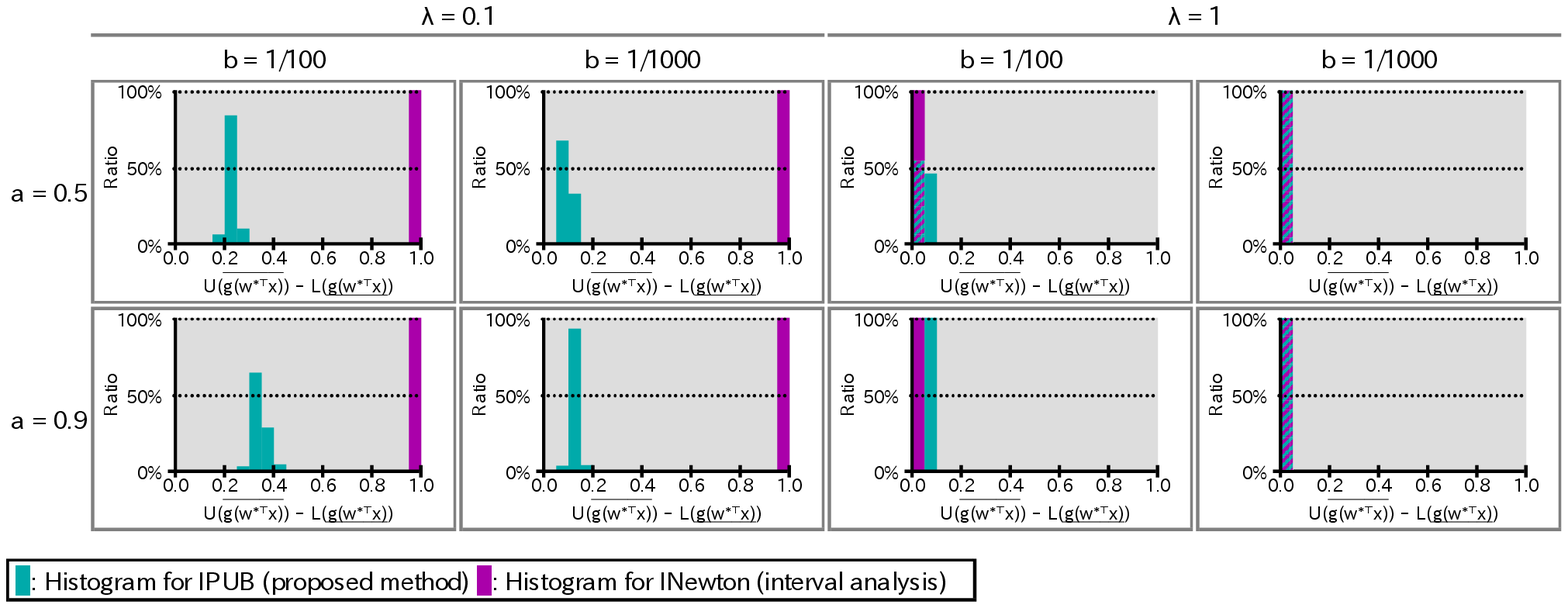} \\
(c) {\tt HIGGS} dataset \\
\caption{
Normalized histograms of prediction uncertainty interval lengths obtained using the IPUB method (green) and the INewton method (purple) for test inputs in three datasets. 
Since both methods guarantee that the prediction values are within the intervals, shorter interval lengths indicate better performance. 
%
%
}
\label{fig:uncertainty-histogram}
\end{center}
\end{figure*}

\paragraph{Computation Time}
Table~\ref{tb:times} lists the ratios of the IPUB to INewton computation times.
%
%
In all cases, the IPUB method is more than 100 times faster than the INewton method. 

\begin{table}[t]
\caption{The ratio of computation times of the IPUB method to those of the INewton method.
}
\label{tb:times}
\begin{center}
{\tabcolsep=1mm {\small
\begin{tabular}{lc|cc|cc}
\hline
& & \multicolumn{2}{c}{$\lambda=0.1$} & \multicolumn{2}{c}{$\lambda=1$} \\
\multicolumn{1}{c}{Dataset} & $a$ & $b=1/100$ & $b=1/1000$ & $b=1/100$ & $b=1/1000$ \\
\hline
{\tt default}        & $0.5$ & 5.16e$-$03 & 5.39e$-$03 & 3.41e$-$03 & 2.40e$-$03 \\
($n=$27,000, $d=$28) & $0.9$ & 4.96e$-$03 & 5.63e$-$03 & 3.27e$-$03 & 2.64e$-$03 \\
\hline
{\tt drive}          & $0.5$ & 2.60e$-$03 & 2.67e$-$03 & 1.33e$-$03 & 1.67e$-$03 \\
($n=$52,658, $d=$48) & $0.9$ & 2.80e$-$03 & 2.89e$-$03 & 1.32e$-$03 & 1.48e$-$03 \\
\hline
{\tt HIGGS}             & $0.5$ & 4.74e$-$03 & 5.10e$-$03 & 1.02e$-$03 & 9.34e$-$04 \\
($n=$1,100,000, $d=$28) & $0.9$ & 4.70e$-$03 & 5.06e$-$03 & 9.68e$-$04 & 8.61e$-$04 \\
\hline
\end{tabular}
} }
\end{center}
\end{table}

\section{Conclusion}
In this paper, we proposed a new method for evaluating prediction uncertainty for learning with missing values. 
The proposed IPUB method is developed based on a novel sensitivity analysis technique for convex optimization problems, and can be applied to a wide class of commonly used machine learning algorithms. 
The IPUB method is associated with less prediction uncertainty and lower computational cost than existing methods.

\section*{Acknowledgment}

This work was partially supported by MEXT KAKENHI (17H00758, 16H06538), JST CREST (JPMJCR1302, JPMJCR1502), RIKEN Center for Advanced Intelligence Project, and JST support program for starting up innovation-hub on materials research by information integration initiative. 

\bibliography{paper-interval}
\bibliographystyle{ieeetr}
\clearpage
\appendix
\section*{Appendix}
\section{Proofs} \label{sec:proof}

We first present a lemma for strongly convex functions
that can be used for finding a set of models
when the data values are changed in the interval.
These proofs are mainly based on convex optimization theory.
%
%
%

\begin{lemma} \label{lm:strong-convexity-sphere}
Let $f(\bm{v})$: $S\to\mathbb{R}$ ($S\subseteq\mathbb{R}^k$: a convex set) be $\lambda$-strongly convex. Let $\bm{v}^* := \argmin_{\bm{v}\in S} f(\bm{v})$ be the optimal solution of the minimization problem.
Then, for any $\bar{\bm{v}}\in S$, $\|\bm{v}^* - \bar{\bm{v}}\|_2\leq\sqrt{(2/\lambda)(f(\bar{\bm{v}}) - f(\bm{v}^*))}$ holds.
\end{lemma}

\begin{proof}
The lemma is proved from the fact that $\partial f(\bm{v}^*)^\top (\bar{\bm{v}} - \bm{v}^*)\geq 0$ holds for any subdifferentiable convex function $f: S\to\mathbb{R}$ for any convex domain $S$ (see B.24 in \cite{bertsekas1999nonlinear} for the proof). Then we have
\begin{align*}
& f(\bar{\bm{v}}) - f(\bm{v}^*) \\
\geq& \partial f(\bm{v}^*)^\top (\bar{\bm{v}} - \bm{v}^*) + (\lambda/2)\|\bar{\bm{v}} - \bm{v}^*\|_2^2
	\quad\text{(strong convexity)}\\
\geq& (\lambda/2)\|\bar{\bm{v}} - \bm{v}^*\|_2^2.
\end{align*}
\end{proof}

\begin{proof}[Proof of Theorem \ref{th:interval-primal-solution}]
For any $X^{\prime\prime}\in[\underline{X}, \overline{X}]$,
let $\bm{w}^{\prime\prime} := \argmin_{\bm{w}} P_{X^{\prime\prime}}(\bm{w})$
and $\bm{\alpha}^{\prime\prime} := \argmin_{\bm{\alpha}} D_{X^{\prime\prime}}(\bm{\alpha})$.
Noting that $P_{X^{\prime\prime}}$ is $\lambda$-strongly convex if $\rho$ is $\lambda$-strongly convex,
we have
\begin{align*}
& \|\bm{w}^{\prime\prime} - \bm{w}^\prime \|_2 \\
\leq & \sqrt{2 [ P_{X^{\prime\prime}}(\bm{w}^\prime) - P_{X^{\prime\prime}}(\bm{w}^{\prime\prime}) ]/\lambda}
	\quad(\because~\text{Lemma \ref{lm:strong-convexity-sphere}}) \\
= & \sqrt{2 [ P_{X^{\prime\prime}}(\bm{w}^\prime) - D_{X^{\prime\prime}}(\bm{\alpha}^{\prime\prime}) ]/\lambda}
	\quad(\because~\text{\eqref{ex:strong-duality}}) \\
\leq & \sqrt{2 [P_{X^{\prime\prime}}(\bm{w}^\prime) - D_{X^{\prime\prime}}(\bm{\alpha}^{\prime})] / \lambda},
	\quad(\because~\text{$\bm{\alpha}^{\prime\prime}$ is a maximizer of $D_{X^{\prime\prime}}$}).
\end{align*}
Therefore we can build a superset of $\cW^*$ as
\begin{align}
& {\cW}^* = \{ \bm{w}^{\prime\prime} \mid X^{\prime\prime}\in[\underline{X}, \overline{X}] \}
	\quad(\because \text{\eqref{eq:set_of_parameter}}) \nonumber \\
\subseteq & \bigcup_{X^{\prime\prime}\in[\underline{X}, \overline{X}]}
	\left\{ \bm{w} \mid \|\bm{w} - \bm{w}^\prime \|_2 \leq \sqrt{2 [P_{X^{\prime\prime}}(\bm{w}^\prime) - D_{X^{\prime\prime}}(\bm{\alpha}^{\prime})] / \lambda} \right\} \nonumber \\
= & \bigl\{ \bm{w} \mid \|\bm{w} - \bm{w}^\prime \|_2 \leq \max_{X^{\prime\prime}\in[\underline{X}, \overline{X}]} \sqrt{2 [P_{X^{\prime\prime}}(\bm{w}^\prime) - D_{X^{\prime\prime}}(\bm{\alpha}^{\prime})] / \lambda} \bigr\} \label{eq:bound-spot2region} \\
& (\text{the union of concentric hyperspheres}) \nonumber \\
\subseteq & \{ \bm{w} \mid \|\bm{w} - \bm{w}^\prime \|_2 \leq \sqrt{2 \Delta / \lambda} \} =: \cW. \label{eq:bound-spot2region-separated}
\end{align}
Thus we know that $\cW$ in \eqref{eq:cW} is a superset of $\cW^*$.
Note that the relationship between \eqref{eq:bound-spot2region} and \eqref{eq:bound-spot2region-separated}
can be proved from the following fact:
\begin{align}
& \max_{X^{\prime\prime}\in[\underline{X}, \overline{X}]} [P_{X^{\prime\prime}}(\bm{w}^\prime) - D_{X^{\prime\prime}}(\bm{\alpha}^{\prime})] \nonumber \\
\leq & \max_{X^{\prime\prime}\in[\underline{X}, \overline{X}]} P_{X^{\prime\prime}}(\bm{w}^\prime) - \min_{X^{\prime\prime}\in[\underline{X}, \overline{X}]} D_{X^{\prime\prime}}(\bm{\alpha}^{\prime}) =: \Delta. \nonumber
\end{align}
\end{proof}

\begin{proof}[Proof of Theorem \ref{th:interval-primal-computation}]
In the setting of the theorem, the following \eqref{ex:max-primal-in-interval} and \eqref{ex:max-dual-in-interval} hold.
Then we have \eqref{eq:duality-gap-upper-bound-limited} by $P_{X^\prime}(\bm{w}^\prime) = D_{X^\prime}(\bm{\alpha}^\prime)$ ($\because$ \eqref{ex:strong-duality}) and Theorem \ref{th:interval-primal-solution}.
\begin{align}
& \max_{X^{\prime\prime}\in[\underline{X}, \overline{X}]} P_{X^{\prime\prime}}(\bm{w}^\prime) - P_{X^\prime}(\bm{w}^\prime) \nonumber\\
= & \frac{1}{n} \sum_{i\in\cI} [\max\{ \ell(y_i, p^-_i), \ell(y_i, p^+_i) \} - \ell(y_i, \bm{w}^{\prime\top}\bm{x}^\prime_{i\cdot})]. \label{ex:max-primal-in-interval}\\
& \min_{X^{\prime\prime}\in[\underline{X}, \overline{X}]} D_{X^{\prime\prime}}(\bm{\alpha}^\prime) - D_{X^\prime}(\bm{\alpha}^\prime) \nonumber\\
= & \sum_{j\in\cJ} \Bigl[ \rho^*_j\left(\frac{1}{n}\bm{\alpha}^{\prime\top}\bm{x}^\prime_{\cdot j}\right) - \max\Bigl\{ \rho^*_j\Bigl(\frac{1}{n} q^-_j \Bigr), \rho^*_j\Bigl(\frac{1}{n} q^+_j \Bigr) \Bigr\} \Bigr]. \label{ex:max-dual-in-interval}
\end{align}

We only present the proof of \eqref{ex:max-primal-in-interval} here;
the counterpart proof for \eqref{ex:max-dual-in-interval} can be similarly derived.
\begin{align}
& \max_{X^{\prime\prime}\in[\underline{X}, \overline{X}]} P_{X^{\prime\prime}}(\bm{w}^\prime) - P_{X^\prime}(\bm{w}^\prime) \nonumber\\
=& \max_{X^{\prime\prime}\in[\underline{X}, \overline{X}]} \frac{1}{n} \sum_{i\in[n]} \ell(y_i, \bm{w}^{\prime\top}\bm{x}^{\prime\prime}_{i\cdot})
	- \frac{1}{n} \sum_{i\in[n]} \ell(y_i, \bm{w}^{\prime\top}\bm{x}^\prime_{i\cdot}) \nonumber\\
=& \frac{1}{n} \sum_{i\in[n]} \max_{\bm{x}^{\prime\prime}_{i\cdot}\in[\underline{\bm{x}_{i\cdot}}, \overline{\bm{x}_{i\cdot}}]} \ell(y_i, \bm{w}^{\prime\top}\bm{x}^{\prime\prime}_{i\cdot})
	- \frac{1}{n} \sum_{i\in[n]} \ell(y_i, \bm{w}^{\prime\top}\bm{x}^\prime_{i\cdot}) \label{eq:bound-primal-decomposed-raw}\\
=& \frac{1}{n} \sum_{i\in\cI} \max_{\bm{x}^{\prime\prime}_{i\cdot}\in[\underline{\bm{x}_{i\cdot}}, \overline{\bm{x}_{i\cdot}}]} \ell(y_i, \bm{w}^{\prime\top}\bm{x}^{\prime\prime}_{i\cdot})
	- \frac{1}{n} \sum_{i\in\cI} \ell(y_i, \bm{w}^{\prime\top}\bm{x}^\prime_{i\cdot})
\label{eq:bound-primal-decomposed} \\
=& \frac{1}{n} \sum_{i\in\cI} \max\{ \ell(y_i, p^-_i), \ell(y_i, p^+_i) \}
	- \frac{1}{n} \sum_{i\in\cI} \ell(y_i, \bm{w}^{\prime\top}\bm{x}^\prime_{i\cdot}).
\label{eq:bound-primal-decomposed-maximized}
\end{align}
The expression \eqref{eq:bound-primal-decomposed-raw} is rewritten as \eqref{eq:bound-primal-decomposed} (fewer terms in the summation) because, if $i\not\in\cI$, then $\underline{\bm{x}_{i\cdot}} = \overline{\bm{x}_{i\cdot}}$ and therefore $\bm{x}^{\prime\prime}_{i\cdot} = \bm{x}^\prime_{i\cdot}$.

The first term in \eqref{eq:bound-primal-decomposed} can be computed as that in \eqref{eq:bound-primal-decomposed-maximized} since $\ell(y, t)$ is a one-variable convex function with respect to $t$; it is maximized at either end of the domain of $t$, that is,
\begin{align*}
t = p^-_i = \min_{\bm{x}^{\prime\prime}_{i\cdot}\in[\underline{\bm{x}_{i\cdot}}, \overline{\bm{x}_{i\cdot}}]} \bm{w}^{\prime\top}\bm{x}^{\prime\prime}_{i\cdot} \quad
\text{or}
\quad
t = p^+_i = \max_{\bm{x}^{\prime\prime}_{i\cdot}\in[\underline{\bm{x}_{i\cdot}}, \overline{\bm{x}_{i\cdot}}]} \bm{w}^{\prime\top}\bm{x}^{\prime\prime}_{i\cdot}.
\end{align*}
The fact that $p^-_i = \min_{\bm{x}^{\prime\prime}_{i\cdot}\in[\underline{\bm{x}_{i\cdot}}, \overline{\bm{x}_{i\cdot}}]} \bm{w}^{\prime\top}\bm{x}^{\prime\prime}_{i\cdot}$ is proved as follows.
$p^+_i$ is similarly proved.
\begin{align}
& \min_{\bm{x}^{\prime\prime}_{i\cdot}\in[\underline{\bm{x}_{i\cdot}}, \overline{\bm{x}_{i\cdot}}]} \bm{w}^{\prime\top}\bm{x}^{\prime\prime}_{i\cdot}
	= \sum_{j\in[d]}(w^\prime_j > 0 ~?~ w^\prime_j \underline{x_{hj}} : w^\prime_j \overline{x_{hj}})\nonumber\\
= & \bm{w}^{\prime\top}\bm{x}^\prime_{i\cdot} + \sum_{j\in[d]}[(w^\prime_j > 0 ~?~ w^\prime_j \underline{x_{hj}} : w^\prime_j \overline{x_{hj}}) - w^\prime_j x^\prime_{ij}] \label{eq:loss-lower-bound-decomposed}\\
= & \bm{w}^{\prime\top}\bm{x}^\prime_{i\cdot} + \sum_{j\in\cV\cJ(i)}[(w^\prime_j > 0 ~?~ w^\prime_j \underline{x_{hj}} : w^\prime_j \overline{x_{hj}}) - w^\prime_j x^\prime_{ij}] \label{eq:loss-lower-bound-decomposed-reduced} \\
= & p^-_i. \nonumber
\end{align}
The expression \eqref{eq:loss-lower-bound-decomposed} is rewritten as
\eqref{eq:loss-lower-bound-decomposed-reduced} following a similar approach to \eqref{eq:bound-primal-decomposed-raw} and \eqref{eq:bound-primal-decomposed} above.
\end{proof}

\begin{proof}[Proof of Corollary \ref{co:bound-inner-product}]
Because we assume $g$ is monotonically non-decreasing,
we only have to prove
\begin{align}
\inf_{\bm{w}\in\cW}\bm{w}^\top\bm{x} = \bm{w}^{\prime\top}\bm{x} - \|\bm{x}\|_2 \sqrt{2 \Delta / \lambda},
	\label{ex:inner-product-inf} \\
\sup_{\bm{w}\in\cW}\bm{w}^\top\bm{x} = \bm{w}^{\prime\top}\bm{x} + \|\bm{x}\|_2 \sqrt{2 \Delta / \lambda}.
	\label{ex:inner-product-sup}
\end{align}
\eqref{ex:inner-product-inf} is proved as follows.
\eqref{ex:inner-product-sup} is similarly proved.
\begin{align*}
&\inf_{\bm{w}\in\cW}\bm{w}^\top\bm{x}
	= \inf_{\bm{w}: \|\bm w - \bm w^\prime\|_2 \le \sqrt{2 \Delta / \lambda}}\bm{w}^\top\bm{x} \\
=& \inf_{\bm{u}: \|\bm u\|_2 \le \sqrt{2 \Delta / \lambda}}(\bm{w}^\prime + \bm{u})^\top\bm{x}
	= \bm{w}^{\prime\top}\bm{x} + \inf_{\bm{u}: \|\bm u\|_2 \le \sqrt{2 \Delta / \lambda}}\bm{u}^\top\bm{x}.
\end{align*}
Because the second term in the last expression is an inner product,
it is minimized when $\|\bm{u}\|_2 = \sqrt{2 \Delta / \lambda}$ and
$\bm{u}$ is directed opposite to $\bm{x}$,
that is, $\bm{u} = -\frac{\bm{x}}{\|\bm{x}\|_2} \sqrt{2 \Delta / \lambda}$
and thus $\bm{u}^\top\bm{x} = - \|\bm{x}\|_2 \sqrt{2 \Delta / \lambda}$.
Therefore, we have \eqref{ex:inner-product-inf}.
\end{proof}

\section{Preprocesses of the datasets for experiments} \label{ch:data-setups}

How to prepare datasets in Section \ref{sec:experiment} is as follows:
\begin{description}
\item[Separation of datasets into training and test sets]
	For each of {\tt default} and {\tt drive} datasets, we took 10\% of instances at random
	as the test set, and the rest as the training set.
	For {\tt HIGGS} dataset, which has 11,000,000 instances in total,
	we took 0.1\% as the test set and 10\% as the training set without overlaps at random.
\item[Outlier removal] For each diversed numerical feature $j$, let $\pi^{(j)}_{0.5}$ be its
	0.5-percentile of the values in the feature of the training dataset.
	Then we replaced all values in feature $j$ of both the training and the test datasets
	smaller than $\pi^{(j)}_{0.5}$ with $\pi^{(j)}_{0.5}$.
	We did the similar for the values in feature $j$ larger than $\pi^{(j)}_{99.5}$
	(99.5-percentile).
\item[Normalization] For each feature, we normalized the values via a linear transformation
	so that the smallest and the largest values in the training dataset are 0 and 1, respectively.
\end{description}


\end{document}